\documentclass{article}

\PassOptionsToPackage{round}{natbib}

\usepackage[final]{neurips_2019}

\usepackage[utf8]{inputenc} 
\usepackage[T1]{fontenc}    
\usepackage{url}            
\usepackage{booktabs}       
\usepackage{amsfonts}       
\usepackage{nicefrac}       
\usepackage{microtype}      
\usepackage{graphicx}

\usepackage{color}
\usepackage{amsmath}
\usepackage{amsfonts}       
\usepackage{varwidth}
\usepackage{xspace}

\usepackage{amsthm}
\usepackage{thmtools}
\usepackage{thm-restate}

\usepackage{subcaption}
\usepackage{floatrow}       
\newfloatcommand{capbtabbox}{table}[][\FBwidth]

\usepackage{algorithm}
\usepackage[noend]{algorithmic}

\usepackage[hidelinks]{hyperref}
\usepackage[capitalise]{cleveref}

\usepackage[export]{adjustbox}
\usepackage{printlen}
\uselengthunit{cm}

\newtheorem{lemma}{Lemma}

\crefname{equation}{}{} 
\crefname{section}{Sec.}{Sec.}
\crefname{algorithm}{Alg.}{Alg.}
\crefname{ALC@unique}{Line}{Lines}

\newcommand{\argmax}{\operatornamewithlimits{argmax}}

\renewcommand{\mid}{\,|\,}

\newcommand{\goose}{\textsc{GoOSE}\xspace}

\newcommand{\s}{\textbf{x}}
\newcommand{\z}{\textbf{z}}
\newcommand{\Rbret}{\tilde{R}^{\mathrm{ret}}}
\newcommand{\Rret}{R^{\mathrm{ret}}}
\newcommand{\Rsafe}{R^{\mathrm{safe}}}
\newcommand{\Rreach}{R^{\mathrm{reach}}}
\newcommand{\Rbreach}{\tilde{R}^{\mathrm{reach}}}
\newcommand{\Rbar}{\tilde{R}}
\newcommand{\Rerg}{R^{\textrm{ergodic}}}
\newcommand{\start}{\s^\dagger}
\newcommand{\goal}{\s^\star}

\newcommand{\Spess}{S^p}

\newcommand{\Shato}{\bar{S}^{o,\epsilon}}
\newcommand{\Shatp}{\bar{S}^p}
\newcommand{\Sadj}{A}
\newcommand{\dom}{\mathcal{D}}
\newcommand{\priority}{\alpha}

\newcommand{\todo}[1]{}

\renewcommand{\paragraph}[1]{\textbf{#1}\hspace{1em}}

\usepackage{glossaries}  
\setacronymstyle{long-short}  
\newacronym{IML}{IML}{Interactive Machine Learning}  

\author{%
  Matteo Turchetta \\
  Dept.~of Computer Science\\
  ETH Zurich \\
  \texttt{matteotu@inf.ethz.ch}
   \and
   Felix Berkenkamp \\
   Dept.~of Computer Science \\
   ETH Zurich\\
   \texttt{befelix@inf.ethz.ch}
   \and
   Andreas Krause \\
   Dept.~of Computer Science \\
   ETH Zurich \\
   \texttt{krausea@ethz.ch}
}

\title{Safe Exploration for Interactive Machine Learning}

\begin{document}

\maketitle


\begin{abstract}
In \gls{IML}, we iteratively make decisions and obtain noisy observations of an unknown function. While \gls{IML} methods, e.g., Bayesian optimization and active learning, have been successful in applications,
on real-world systems they must provably avoid unsafe decisions. To this end, safe \gls{IML} algorithms must carefully learn about \textit{a priori} unknown constraints without making unsafe decisions. 
Existing algorithms for this problem learn about the safety of all decisions to ensure convergence. 
This is sample-inefficient, as it explores decisions that are not relevant for the original \gls{IML} objective.  
In this paper, we introduce a novel framework that renders any existing unsafe \gls{IML} algorithm safe. Our method works as an add-on that takes suggested decisions as input and exploits regularity assumptions in terms of a Gaussian process prior in order to efficiently learn about their safety. As a result, we only explore the safe set when necessary for the \gls{IML} problem. We apply our framework to safe Bayesian optimization and to safe exploration in deterministic Markov Decision Processes (MDP), which have been analyzed separately before. Our method outperforms other algorithms empirically.
\end{abstract}

\glsresetall

\section{Introduction} 
\label{sec:introduction}

\gls{IML} problems, where an autonomous agent actively queries an unknown function to optimize it, learn it, or otherwise act based on the observations made, are pervasive in science and engineering. For example, Bayesian optimization (BO) \citep{mockus1978application} is an established paradigm to optimize unknown functions and has been applied to diverse tasks such as optimizing robotic controllers~\citep{marco2017virtual} and hyperparameter tuning in machine learning~\citep{snoek2012practical}. Similarly, Markov Decision Processes (MDPs) \citep{puterman2014markov} model sequential decision making problems with long term consequences and are applied to a wide range of problems including finance and management of water resources \citep{white1993survey}. 

However, real-world applications are subject to safety constraints, which cannot be violated during the learning process. Since the dependence of the safety constraints on the decisions is unknown \textit{a priori}, existing algorithms are not applicable. To \textit{optimize} the objective without violating the safety constraints, we must \textit{carefully explore} the space and ensure that decisions are safe before evaluating them. In this paper, we propose a data-efficient algorithm for safety-constrained \gls{IML} problems.

\paragraph{Related work}
One class of \gls{IML} algorithms that consider safety are those for BO with Gaussian Process (GP) \citep{rasmussen2004gaussian} models of the objective. While classical BO algorithms focus on efficient optimization~\citep{srinivas2009gaussian,thompson1933likelihood,wang2017max}, these methods have been extended to incorporate safety constraints. For example, \citet{gelbart2014bayesian} present a variant of expected improvement with unknown constraints, while \citet{hernandez2016general} extend an information-theoretic BO criterion to handle black-box constraints. However, these methods only consider finding a safe solution, but allow unsafe evaluations during the optimization process. \citet{Wu2016ConservativeB} define safety as a constraint on the cumulative reward, while  \citet{schreiter2015safe} consider the safe exploration task on its own. The algorithms \textsc{SafeOPT} \citep{sui2015safe,berkenkamp2016bayesian} 
and \textsc{StageOPT} \citep{sui2018stagewise} both guarantee safety of the exploration and near-optimality of the solution. 
However, they treat the exploration of the safe set as a proxy objective, which leads to sample-inefficient exploration as they explore the entire safe set, even if this is not necessary for the optimization task, see the evaluation counts (green) in \cref{fig:illustrative_stageopt} for an example.

\begin{figure*}
    \begin{subfigure}[b]{0.35\textwidth}
        \includegraphics{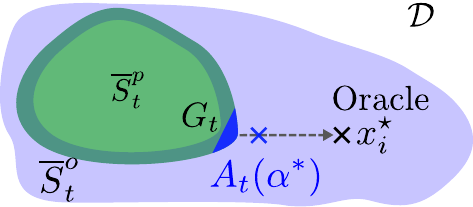}
        \vspace{0.7em}
        \caption{Set illustration.}
        \label{fig:set_illustration}
    \end{subfigure}%
    \begin{subfigure}[b]{0.32\textwidth}
        \includegraphics{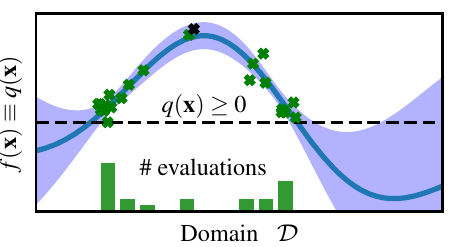} 
        \caption{\textsc{StageOPT}.}
        \label{fig:illustrative_stageopt}
    \end{subfigure}%
    \begin{subfigure}[b]{0.32\textwidth}
        \includegraphics{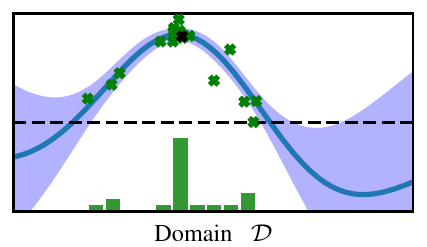}
        \caption{\textsc{GP-UCB} + \goose (ours)}
        \label{fig:illustrative_goose}
    \end{subfigure}
    \caption{Existing algorithms for safe \gls{IML} aim to expand the safe set $\Shatp$ (green shaded) in \cref{fig:set_illustration} by evaluating decisions on the boundary of the pessimistic safe set (dark green shaded). This can be inefficient: to solve the safe BO problem in \cref{fig:illustrative_stageopt}, \textsc{StageOPT} evaluates decisions (green crosses, histogram) close to the safety constraint $q(\cdot) > 0$ (black dashed), even though the maximum (black cross) is known to be safe. In contrast, our method uses decisions $\goal_i$ from existing \textit{unsafe} \gls{IML} algorithms (oracle) within the optimistic safe set $\Shato_t$ (blue shaded, \cref{fig:set_illustration}). It can then use any heuristic to select learning targets $A_t$ (blue cross) that are informative about the safety of $\goal_i$ and learns about them efficiently within $G_t \subseteq \Shatp_t$ (blue shaded region). Since this method only learns about the safe set when necessary, we evaluate more close-to-optimal decisions in \cref{fig:illustrative_goose}.
    }
    \label{fig:IllustrativeExample}
\end{figure*}

Safety has also been investigated in \gls{IML} problems in directed  graphs, where decisions have long-term effects in terms of safety. \citet{moldovan2012safe} address this problem in the context of discrete MDPs by optimizing over ergodic policies, i.e., policies that are able to return to a known set of safe states with high probability. However, they do not provide exploration guarantees. \citet{biyik2019efficient} study the ergodic exploration problem in discrete and deterministic MDPs with unknown dynamics and noiseless observations.  \citet{turchetta2016safe} investigate the ergodic exploration problem subject to unknown external safety constraints under the assumption of known dynamics by imposing additional ergodicity constraints on the  \textsc{SafeOPT} algorithm. \citet{wachi2018safe} compute approximately optimal policies in the same context but do not actively learn about the constraint. In continuous domains, safety has been investigated by, for example, \citet{akametalu2014reachability,koller18safempc}. While these methods provide safety guarantees, current exploration guarantees rely on uncertainty sampling on a discretized domain~\citep{berkenkamp2017safe}. Thus, their analysis can benefit from the more efficient, goal-oriented exploration introduced in this paper.

\paragraph{Contribution} 
In this paper, we introduce the Goal Oriented Safe Exploration algorithm, \goose; a novel framework that works as an add-on to existing \gls{IML} algorithms and renders them safe. Given a possibly unsafe suggestion by an \gls{IML} algorithm, it safely and efficiently learns about the safety of this decision by exploiting continuity properties of the constraints in terms of a GP prior. Thus, unlike previous work, \goose only learns about the safety of decisions relevant for the \gls{IML} problem. We analyze our algorithm and prove that, with high probability, it only takes safe actions while learning about the safety of the suggested decisions. On safe BO problems, our algorithm leads to a bound on a natural notion of safe cumulative regret when combined with a no-regret BO algorithm. Similarly, we use our algorithm for the safe exploration in deterministic MDPs. Our experiments show that \goose is significantly more data-efficient than existing methods in both settings.


\section{Problem Statement and Background}
\label{sec:ProblemStatement}

In \gls{IML}, an agent iteratively makes decisions and observes their consequences, which it can use to make better decisions over time. Formally, at iteration $i$, the agent $\mathcal{O}_i$ uses the previous $i-1$ observations to make a new decision $\goal_i = \mathcal{O}_i(\mathcal{D}_i)$ from a finite decision space $\mathcal{D}_i \subseteq \mathcal{D} \subseteq \mathbb{R}^d$. It then observes a noisy measurement of the unknown objective function $f:\mathcal{D} \rightarrow \mathbb{R}$ and uses the new information in the next iteration. This is illustrated in the top-left corner (blue shaded) of \cref{fig:alg_explanation}.
Depending on the goal of the agent, this formulation captures a broad class of problems and many solutions to these problems have been proposed. For example, in \emph{Bayesian optimization} the agent aims to find the global optimum $\max_\s f(\s)$ \citep{mockus1978application}. Similarly, in active learning \citep{schreiter2015safe}, one aims to learn about the function $f$. In the general case, the decision process may be stateful, e.g., as in dynamical systems, so that the decisions $\mathcal{D}_i$ available to the agent depend on those made in the past. This dependency among decisions can be modeled with a \textit{directed graph}, where nodes represent decisions and an edge connects node $\s$ to node $\s'$ if the agent is allowed to evaluate $\s'$ given that it evaluated $\s$ at the previous decision step. In the BO setting, the graph is fully-connected and any decision may be evaluated, while in a deterministic MDP decisions are states and edges represent transitions \citep{turchetta2016safe}.

In this paper, we consider \gls{IML} problems with safety constraints, which frequently occur in real-world settings. The safety constraint can be written as $q(\s) \geq 0$ for some function $q$. Any decision $\goal_i$ for $i \geq 1$ evaluated by the agent must be safe. For example, \citet{berkenkamp2016bayesian} optimize the control policy of a flying robot and must evaluate only policies that induce trajectories satisfying given constraints. However, it is unknown \textit{a priori} which policy parameters induce safe trajectories. Thus, we do not know which decisions are safe in advance, that is, $q:\mathcal{D} \to \mathbb{R}$ is \textit{a priori} unknown. However, we can learn about the safety constraint by selecting decisions $\s_t$ and obtaining noisy observations of $q(\s_t)$. We denote queries to $f$ with $\goal_i$ and queries to $q$ with $\s_t$. As a result, we face a two-tiered \textit{safe exploration} problem: On one hand we have to safely learn about the constraint $q$ to determine which decisions are safe, while on the other hand we want to learn about $f$ to solve the \gls{IML} problem. The goal is to minimize the number of queries $\s_t$ required to solve the \gls{IML} problem.

\begin{figure*}[t]
    \includegraphics[width=0.65\textwidth]{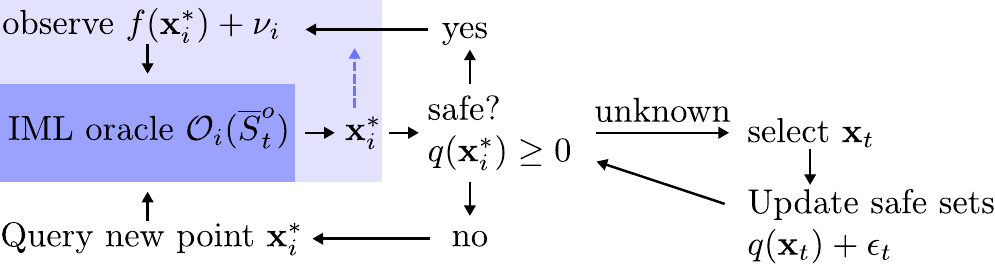}
    \caption{Overview of \goose. If the oracle's suggestion $\goal_i$ is safe, it can be evaluated. This is equivalent to the standard \textit{unsafe} \gls{IML} pipeline (top-left, blue shaded) in \cref{fig:alg_explanation}. Otherwise, \goose learns about the safety of $\goal_i$ by actively querying observations at decisions $\s_t$. Any provably unsafe decision is removed from the decision space and we query a new $\goal_i$ without providing a new observation of $f(\goal_i)$.    
    }
    \label{fig:alg_explanation}
\end{figure*}

\paragraph{Regularity}
Without further assumptions, it is impossible to evaluate decisions without violating the safety constraint $q$ \citep{sui2015safe}. For example, without an initial set of decisions that is known to be safe \textit{a priori}, we may fail at the first step. Moreover, if the constraint does not exhibit any regularity, we cannot infer the safety of decisions without evaluating them first.
We assume that a small initial safe set of decisions, $S_0$, is available, which may come from domain knowledge. 
Additionally, we assume that $\mathcal{D}$ is endowed with a positive definite kernel function, $k(\cdot, \cdot)$, and that the safety constraint $q$ has bounded norm in the induced \textit{Reproducing Kernel Hilbert Space} (RKHS)  \citep{scholkopf2002learning}), $\|q\|_k \leq B_q$. The RKHS norm measures the smoothness of the safety feature with respect to the kernel, so that $q$ is $L$-Lipschitz continuous with respect to the kernel metric $d(\s, \s') = \sqrt{k(\s, \s) - 2k(\s, \s') + k(\s', \s')}$ with $L=B_q$ \citep[(4.21)]{steinwart2008support}

This assumption allows us to model the safety constraint function $q$ with a GP \citep{rasmussen2004gaussian}. A GP is a distribution over functions parameterized by a mean function $\mu(\cdot)$ and a covariance function $k(\cdot, \cdot)$. We set $\mu(\s)=0$ for all $\s \in \mathcal{D}$ without loss of generality. The covariance function encodes our assumptions about the safety constraint. Given $t$ observations of the constraint $\textbf{y}=(q(\s_1) + \eta_1, \dots, q(\s_t) + \eta_t )$ at decisions $\mathcal{D}_t=\{\s_n\}_{n=1}^t$, where $\eta_n \sim \mathcal{N}(0, \sigma^2)$ is a zero-mean i.i.d. Gaussian noise, the posterior belief is distributed as a GP with mean, covariance, and variance
\begin{equation*}
\mu_t(\s) = \mathbf{k}_t^\mathrm{T}(\s) (\mathbf{K}_t + \sigma^2 \mathbf{I})^{-1} \mathbf{y}_t, ~
    k_t(\s,\s') = k(\s,\s') - \mathbf{k}_t^\mathrm{T}(\s) (\mathbf{K}_t + \sigma^2 \mathbf{I})^{-1} \mathbf{k}_t(\s'), ~
    \sigma_t(\s) = k_t(\s, \s)
\end{equation*}
respectively. Here, ${\mathbf{k}_t(\s)=(k(\s_1,\s), \dots, k(\s_t,\s))}$, $\mathbf{K}_t$ is the positive definite kernel matrix $[k(\s,\s')]_{\s,\s' \in D_t}$, and $\mathbf{I} \in \mathbb{R}^{t \times t}$ denotes the identity matrix.

\paragraph{Safe decisions}
The previous regularity assumptions can be used to determine which decisions are safe to evaluate. Our classification of the decision space is related to the one by \citet{turchetta2016safe}, which combines non-increasing and reliable confidence intervals on~$q$ with a reachability analysis of the underlying graph structure for decisions. 
Based on a result by \citet{chowdhury2017kernelized}, they use the posterior GP distribution to construct confidence bounds $l_t(\s) := \max(l_{t-1}(\s), \mu_{t-1}(\s) - \beta_t \sigma_{t-1}(\s))$ and $u_t(\s) := \min(u_{t-1}(\s), \mu_{t-1}(\s) + \beta_t \sigma_{t-1}(\s))$ on the function $q$. 
In particular, we have $l_t(\s) \leq q(\s) \leq u_t(\s)$ with high probability when the scaling factor $\beta_t$ is chosen as in \cref{thm:safety_and_complteness}. Thus, any decision $\s$ with $l_t(\s) \geq 0$ is satisfies the safety constraint $q(\s) \geq 0$ with high probability.


To analyze the exploration behavior of their algorithm, \citet{turchetta2016safe} use the confidence intervals within the current safe set, starting from~$S_0$, and the  Lipschitz continuity of $q$ to define $\Spess_t$, the set of decisions that satisfy the  constraint with high probability. We use a similar, albeit more efficient, definition in \cref{sec:Algorithm}. In practice, one may use the confidence intervals directly. Moreover, in order to avoid exploring decisions that are instantaneously safe but that would force the agent to eventually evaluate unsafe ones due to the graph structure $\mathcal{G}$, \citet{turchetta2016safe} define $\Shatp_t$, the subset of safe and ergodic decisions, i.e., decisions that are safe to evaluate in the short and long term.


\paragraph{Previous Exploration Schemes}
Given that only decisions in $\Shatp_t$ are safe to evaluate, any \textit{safe} \gls{IML} algorithm faces an extended exploration-exploitation problem: it can either optimize decisions within $\Shatp_t$, or expand the set of safe decisions in $\Shatp_t$ by evaluating decisions on its boundary.
Existing solutions to the safe exploration problem in both discrete and continuous domains either do not provide theoretical exploration guarantees \citep{wachi2018safe} or treat the exploration of the safe set as a proxy objective for optimality. That is, the methods uniformly reduce uncertainty on the boundary of the safe set in \cref{fig:set_illustration} until the entire safe set is learned. Since learning about the entire safe set is often unnecessary for the \gls{IML} algorithm, this procedure can be sample-inefficient. For example, in the safe BO problem in \cref{fig:illustrative_stageopt} with $f=q$, this exploration scheme leads to a large number of unnecessary evaluations on the boundary of the safe set.


\section{Goal-oriented Safe Exploration (\goose)} 
\label{sec:Algorithm}

In this section, we present our algorithm, \goose. We do not propose a new \textit{safe} algorithm for a specific \gls{IML} setting, but instead exploit that, for specific \textit{IML} problems high-performance, \textit{unsafe} algorithms already exist. We treat any such \textit{unsafe} algorithm as an \gls{IML} oracle $\mathcal{O}_i(S)$, which, given a domain $S$ and $i-1$ observations of $f$, suggests a new decision $\goal_i \in S$, see \cref{fig:alg_explanation} (blue shaded).

\goose can extend any such \textit{unsafe} \gls{IML} algorithm to the safety-constrained setting. Thus, we effectively leave the problem of querying $f$ to the oracle and only consider safety. Given an \textit{unsafe} oracle decision $\goal_i$, \goose only evaluates $f(\goal_i)$ if the decisions $\goal_i$ is known to be safe. Otherwise, it \textit{safely learns} about $q(\goal_i)$ by safely and efficiently collecting observations $q(\s_t)$. Eventually it either learns that the decision $\goal_i$ is safe and allows the oracle to evaluate $f(\goal_i)$, or that $\goal_i$ cannot be guaranteed to be safe given an $\epsilon$-accurate knowledge of the constraint, in which case the decision set of the oracle is restricted and a new decision is queried, see \cref{fig:alg_explanation}.
\todo{Maybe add here the part about safe, unsafe, undecided}

Previous approaches treat the expansion of the safe set as a proxy-objective to provide completeness guarantees. Instead, \goose employs goal-directed exploration scheme with a novel theoretical analysis that shifts the focus from greedily reducing the uncertainty \textit{inside} the safe set to learning about the safety of decisions \textit{outside} of it. This scheme retains the worst-case guarantees of existing methods, but is significantly more sample-efficient in practice. Moreover, \goose encompasses existing methods for this problem. 
We now describe the detailed steps of \goose in \cref{alg:goose,alg:se}.

\paragraph{Pessimistic and optimistic expansion.} 
To effectively shift the focus from inside the safe set to outside of it, \goose must reason not only about the decisions that are currently known to be safe but also about those that could eventually be classified as safe in the future.
In particular, it maintains two sets, which are an inner/outer approximation of the set of safe decisions that are reachable from $S_0$ and are based on a pessimistic/optimistic estimate of the constraint given the data, respectively. 

The pessimistic safe set contains the decisions that are safe with high probability and is necessary to guarantee safe exploration \citep{turchetta2016safe}. It is defined in two steps: discarding the decisions that are not instantaneously safe and discarding those that we cannot reach/return from safely (see \cref{fig:long_term_safety}) and, thus, are not safe in the long term. To characterize it starting from a given set of safe decisions $S$, we define the pessimistic constraint satisfaction operator,
\begin{equation}
    p_t(S) = \{\s \in \dom, \,\vert\, \exists \z \in S:l_t(\z)-L d(\s,\z) \geq 0\}, 
    \label{eq:pessimistic_expansion}
\end{equation}
which uses the lower bound on the safety constraint of the decisions in $S$ and the Lipschitz continuity of $q$ to determine the decisions that instantaneously satisfy the constraint with high probability, see \cref{fig:expansion_example}.
However, for a general graph $\mathcal{G}$, decisions in $p_t(S)$ may be unsafe in the long-term as in \cref{fig:long_term_safety}: No safe path to the decision $\s_5$ exists, so that it can not be safely \textit{reached}. Similarly, if we were to evaluate $\s_4$, the graph structure forces us to eventually evaluate $\s_3$, which is not contained in $p_t(S)$ and might be unsafe. That is, we cannot safely \textit{return} from $\s_4$. To exclude these decisions, we use the ergodicity operator introduced by \citet{turchetta2016safe}, which allows us to find those decisions that are pessimistically safe in the short and long term $P_t^{1}(S) =p_t(S) \cap \Rerg(p_t(S), S)$ (see \cref{app:definitions} or \citep{turchetta2016safe} for the definition of $R^{\mathrm{ergodic}}$).
%
Alternating these operations $n$ times, we obtain the $n$-step pessimistic expansion operator, $P_t^{n}(S) = p_t(P_t^{n-1}(S)) \cap \Rerg(p_t(P_t^{n-1}(S)), S)$, which, after a finite number of steps, converges to its limiting set $\tilde{P}_t(S)=\lim_{n\to \infty }P_t^{n}(S)$. 

The optimistic safe set excludes the decisions that are unsafe with high probability and makes the exploration efficient by restricting the decision space of the oracle. Similarly to the pessimistic one, it is defined in two steps. However, it uses the following optimistic constraint satisfaction operator,
\begin{equation}
    o_t^\epsilon(S) = \{\s \in \dom, \,\vert\, \exists \z \in S:u_t(\z)-L d(\s,\z) - \epsilon\geq 0\}.
\end{equation}
 See \cref{fig:expansion_example} for a graphical intuition. The additional $\epsilon$-uncertainty term in the optimistic operator accounts for the fact that we only have access to noisy measurements of the constraint and, therefore, we can only learn it up to a specified statistical accuracy. The definitions of the optimistic expansion operators  $O_t^{\epsilon,n}(S)$ and $\tilde{O}_t^\epsilon(S)$ are analogous to the pessimistic case by substituting $p_t$ with $ o_t^\epsilon$.
The sets $\tilde{P}_t$ and $\tilde{O}_t^\epsilon$ indicate the largest set of decisions that can be classified as safe in the short and long term assuming  the constraint attains the worst/best possible value within $S$, given the observations available and, for the optimistic case, despite an $\epsilon$ uncertainty.
\begin{figure*}[t]
    \begin{subfigure}[l]{0.5\textwidth}
        \includegraphics[scale=0.95]{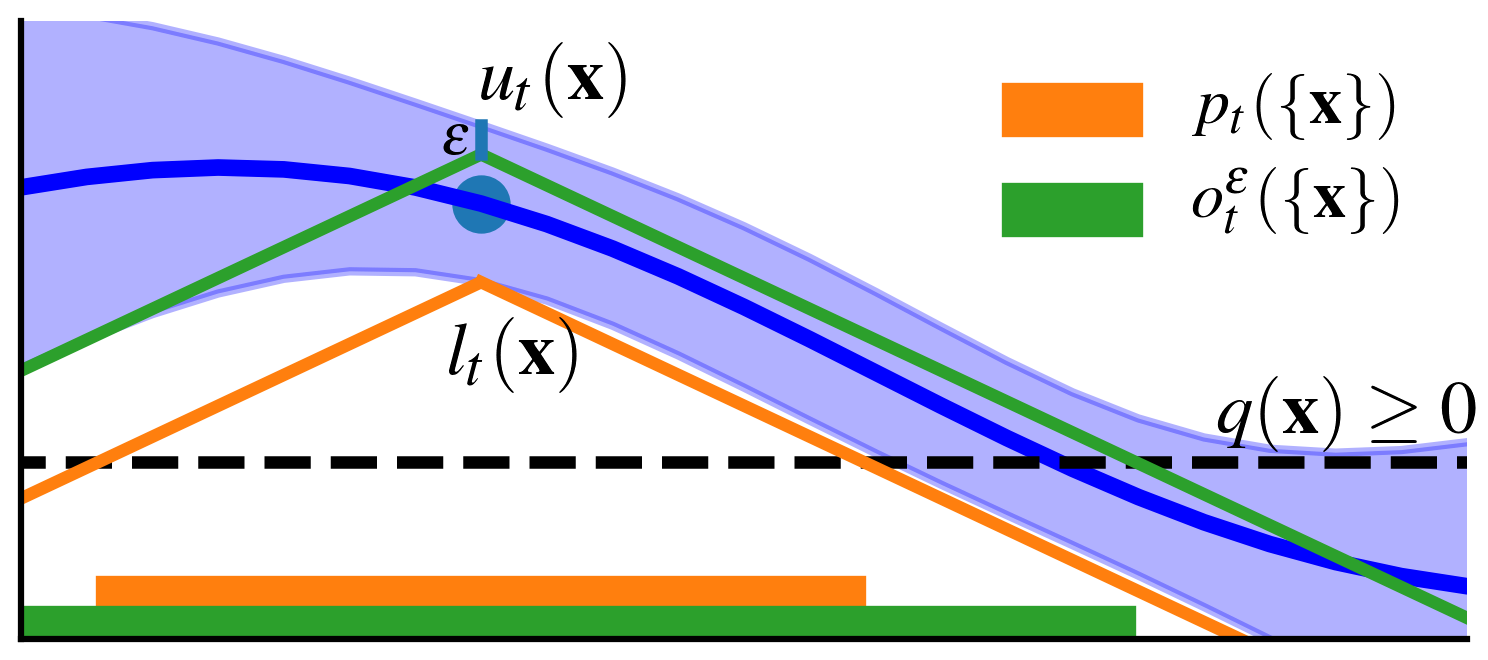}
        \caption{Expansion operators.}
        \label{fig:expansion_example}
    \end{subfigure}%
    \begin{subfigure}{0.35\textwidth}
        \includegraphics{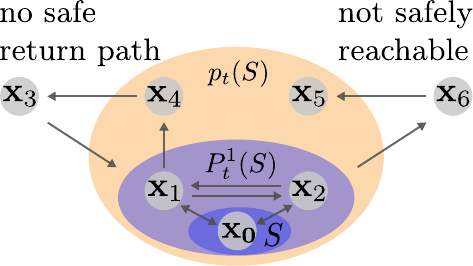}
        \caption{Long term safety in graph.}
        \label{fig:long_term_safety}
    \end{subfigure}%
    \caption{\cref{fig:expansion_example} shows the pessimistic and optimistic constraint satisfaction operators that use the confidence intervals on the constraint and its Lipschitz continuity to make inference about the safety of decisions that have not yet been evaluated. \cref{fig:long_term_safety} illustrates the long-term safety definition. While decisions in $p_t(S)$ are myopically safe, decisions in $P^1_t(S)$ are safe in the long-term. This excludes $\s_4$ and $\s_5$, as no safe path from/to them exists.}
    \label{fig:short_long_term_safety}
\end{figure*}

\paragraph{Optimistic oracle}
The \gls{IML} oracle $\mathcal{O}_i(S)$ suggests  decisions $\goal_i \in S$ to evaluate within a given subset $S$ of $\mathcal{D}$. To make the oracle efficient, we restrict its decision space to decisions that could optimistically be safe in the long and short term. In particular, we define the optimistic safe set $\Shato$ in \cref{alg:goose:line:intersect_sopt_with_ergodic} of \cref{alg:goose} based on the optimistic expansion operator introduced above. The oracle uses this set to suggest a potentially unsafe, candidate decision $\goal_i = \mathcal{O}_i(\Shato_t)$ in \cref{alg:goose:line:oracle_suggestion}. 

\paragraph{Safe evaluation}
We determine safety of the suggestion $\goal_i$ similarly to \citet{turchetta2016safe} by constructing the set $\Shatp_t$ of decisions that are safe to evaluate. However, while \cite{turchetta2016safe} use the one step pessimistic expansion operator in their definition, $P^{1}_t$, we use its limit set in \cref{alg:goose:line:Shatp} of \cref{alg:goose}, $\tilde{P}_t$. While both operators eventually identify the same safe set, our definition allows for a more efficient expansion. For example, consider the case where the graph over the decision space $\mathcal{G}$ is a chain of length $m$ and where, for all $j=1,\cdots,m$, the lower bound on the safety of decision $j-1$ guarantees the safety of decision $j$ with high probability. In this case, \cite{turchetta2016safe} require $m-1$ iterations to fully expand the safe set, while our classification requires only one. 

If we know that $\goal_i$ is safe to evaluate, i.e., $\goal_i \in \Shatp_t$, then the oracle obtains a noisy observation of $f(\goal_i)$ in \cref{alg:goose:line:evaluate_suggestion}. Otherwise \goose proceeds to safely learn about the safety of $\goal_i$ using a safe expansion strategy in lines \crefrange{alg:goose:line:while_unsafe}{alg:goose:line:intersect_sopt_with_ergodic} that we outline in the following. This routine is repeated until we can either include $\goal_i$ in $\Shatp_t$, in which case we can safely evaluate $f(\goal_i$), or remove it from the decision space $\Shato_t$ and query the oracle for a new suggestion. 
%

\begin{figure*}
\noindent
\begin{minipage}[t]{0.53\textwidth}
\centering
\begin{algorithm}[H]
\caption{\goose}
\label{alg:SMDP_priority}
\label{alg:goose}
\begin{algorithmic}[1]
  \algsetup{indent=1em}
  \STATE \textbf{Inputs:} \parbox[t]{0.75\linewidth}{%
          Lipschitz constant $L$, 
          Seed $S_0$, \newline Graph $\mathcal{G}$, Oracle $O$, Accuracy $\epsilon$.
            }
    \STATE $\Shatp_0 \gets S_0$, $\Shato_0 \gets \mathcal{D}, t \gets 0$, \\$l_0(\s)\gets 0$ for $\s \in S_0$
    \FOR{$k=1,2,\ldots$}
        \STATE $\goal_i \gets \mathcal{O}(\Shato_t)$
        \label{alg:goose:line:oracle_suggestion}
        \WHILE{$\goal_i \not \in \Shatp_t$}
        \label{alg:goose:line:while_unsafe}
            \STATE $\mathrm{SE}(\Shato_t, \Shatp_t, \mathcal{G},\goal_i)$, $t\gets t+1$
            \label{alg:goose:line:safe_exploration}
            %
            %
            \STATE $\Shatp_t \gets \tilde{P}_t(\Shatp_{t-1})$
            \label{alg:goose:line:Shatp}
            %
            %
            \STATE $\Shato_t \gets \tilde{O}_t^\epsilon(\Shatp_{t-1})$
            \label{alg:goose:line:intersect_sopt_with_ergodic}
            %
            \STATE \label{alg:goose:line:jumpto_oracle_query} \textbf{if} $\goal_i \not \in \Shato_t$ \textbf{then go to} \cref{alg:goose:line:oracle_suggestion}
        \ENDWHILE
        \STATE Evaluate $f(\goal_i)$ and update oracle 
        \label{alg:goose:line:evaluate_suggestion}
    \ENDFOR
\end{algorithmic}
\end{algorithm}
\end{minipage}
\hfill
\begin{minipage}[t]{0.46\textwidth}
\centering
\begin{algorithm}[H]
\caption{Safe Expansion (SE)}
\label{alg:SE}
\label{alg:se}
\begin{algorithmic}[1]
  \algsetup{indent=1em}
  \STATE {\bfseries Inputs:} $\Shato_t$, $\Shatp_t$,\,$\mathcal{G}$,\,$\goal$
    \STATE \label{alg:se:line:uncertain_states_W} $W_t^\epsilon \gets \{ \s \in \Shatp_t \mid u_t(\s) - l_t(\s) > \epsilon \}$
    \STATE \label{alg:se:line:adjacent} 
    $\Sadj_t(p) \gets \{ \s \in \Shato_t \setminus p^0_t(\Shatp_t) \mid h(\s) = p \}$
    %
    \STATE \label{alg:se:line:expansion_target}
    {\small // Highest priority targets in $\Sadj_t$ with expanders}
        $\priority^* \gets \max \priority \text{~~s.t.~~}  \left| G_t^\epsilon(\priority) \right| > 0$
    %
    \IF{\small optimization problem feasible}
        %
        \STATE \label{alg:se:line:uncertainty_sampling}
        $\s_t \gets \argmax_{\s \in G_t^\epsilon(\priority^*)} \, w_t(\s)$
        \STATE \label{alg:se:line:update_gp} 
        Update $\mathrm{GP}$ with $y_t = q(\s_t) + \eta_t$ 
    \ENDIF
    %
\end{algorithmic}
\end{algorithm}
\end{minipage}
\end{figure*}

\paragraph{Safe expansion} 
If the oracle suggestion $\goal_i$ is not considered safe,  $\goal_i \notin \Shatp_t$, \goose employs a goal-directed scheme to evaluate a safe decision $\s_t \in \Shatp_t$ that is informative about $q(\goal_i)$, see \cref{fig:set_illustration}.
In practice, it is desirable to avoid learning about decisions beyond a certain accuracy $\epsilon$, as the number of observations required to reduce the uncertainty grows exponentially with $\epsilon$ \citep{sui2018stagewise}. Thus, we only learn about decisions in $\Shatp_t$ whose safety values are not known $\epsilon$-accurately yet in Line \ref{alg:se:line:uncertain_states_W}, $W_t^\epsilon = \{\s \in \Shatp_t \mid u_t(\s) - l_t(\s) > \epsilon\}$, where $u_t(\s) - l_t(\s)$ is the width of the confidence interval at $\s$.

To decide which decision in $W_t^\epsilon$ to learn about, we first determine a set of learning targets outside the safe set (dark blue cross in \cref{fig:set_illustration}), and then learn about them efficiently within $\Shatp_t$. To quantify how useful a learning target $\s$ is to learn about $q(\goal_i)$, we use any given iteration-dependent heuristic $h_t(\s)$. We discuss particular choices later, but a large priority $h(\s)$ indicates a relevant learning target (dashed line, \cref{fig:set_illustration}). Since $p^0_t(\Shatp_{t})$ denotes the decisions that are known to satisfy the constraint with high probability and $\Shato_t$ excludes the decisions that are unsafe with high probability, $\Shato_t \setminus p^0_t(\Shatp_{t})$ indicates the decisions whose safety we are uncertain about. We sort them according to their priority and let $\Sadj_t(\priority)$ denote the subset of decision with equal priority.  

Ideally, we want to learn about the decisions with the highest priority. However, this may not be immediately possible by evaluating decisions within $W_t^\epsilon$. Thus, we must identify the decisions with the highest priority that we can learn about starting from $W_t^\epsilon$. Therefore, similarly to the definition of the optimistic safe set, we identify decisions $\s$ in $W_t^\epsilon$ that have a large enough plausible value $q(\s)$ that they could guarantee that $q(\z)\geq 0$ for some $\z$ in $\Sadj_t(\priority)$. However, in this case, we are only interested in decisions that can be instantly classified as safe (rather than eventually). Therefore, we focus on this set of \textit{potential immediate expanders}, 
$
    G_t^\epsilon(\priority)=\{\s \in W_t^\epsilon, \mid \exists\, \z \in \Sadj_t(\priority) \colon u_t(\s) - L d(\s, \z) \geq 0 \}.
$
%
In Line \ref{alg:se:line:expansion_target} of \cref{alg:SE} we select the decisions with the priority level $\priority^*$ such that there exist uncertain, safe decisions in $W_t^\epsilon$ that could allow us to classify a decision in $\Sadj_t(\priority^*)$ as safe and thereby expand the current safe set $\Shatp_t$. Intuitively, we look for the highest priority targets that can potentially be classified as safe by safely evaluating decisions that we have not already learned about to $\epsilon$-accuracy.

Given these learning targets $\Sadj_t(\priority^*)$ (blue cross, \cref{fig:set_illustration}), we evaluate the most uncertain decision in $G_t^\epsilon(\priority^*)$ (blue shaded, \cref{fig:set_illustration}) in Line \ref{alg:se:line:uncertainty_sampling} and update the GP model with the corresponding observation of $q(\s_t)$ in Line \ref{alg:se:line:update_gp}. This uncertainty sampling is restricted to a small set of decisions close to the goal. This is different from methods without a heuristic that select the most uncertain secision on the boundary of $\Shatp$ (green shaded in \cref{fig:set_illustration}). In fact, our method is equivalent to the one by \citet{turchetta2016safe} when an uninformative heuristic $h(\s) = 1$ is used for all $\s$.
We iteratively select and evaluate decisions $\s_t$ until we either determine that $\goal_i$ is safe, in which case it is added to $\Shatp$, or we prove that we can not safely learn about it for given accuracy $\epsilon$, in which case is removed from $\Shato$ and a the oracle is queried with an updated decision space for a new suggestion.

To analyze our algorithm, we define the largest set that we can learn about as $\Rbar_\epsilon(S_0)$. This set contains all the decisions that we could certify as safe if we used a full-exploration scheme that learns the safety constraint $q$ up to $\epsilon$ accuracy for all decisions inside the current safe set. This is a natural exploration target for our safe exploration problem (see \cref{app:definitions} for a formal definition).  We have the following main result, which holds for any heuristic:
\begin{restatable}{theorem}{mainresult}
Assume that $q(\cdot)$ is $L$-Lipschitz continuous w.r.t. $d(\cdot, \cdot)$ with ${\| q \|_k \leq B_q}$, $\sigma$-sub-Gaussian noise, ${S_0 \neq \emptyset}$, ${q(\s) \geq 0}$ for all~${\s \in S_0}$, and that, for any two decisions~${\s,\s' \in S_0}$, there is a path in the graph $\mathcal{G}$ connecting them within $S_0$. Let~$\beta_t^{1/2}=B_q+4\sigma \sqrt{\gamma_t + 1 + \mathrm{ln}(1/\delta)}$, then, for any $h_t:\mathcal{D}\rightarrow \mathbb{R}$, with probability at least ${1-\delta}$, we have
${
q(\s)\geq 0
}$
for any $\s$ visited by~\goose.
Moreover, let $\gamma_t$ denote the information capacity associated with the kernel $k$ and let~$t^*$ be the smallest integer such that
${
\frac{t^*}{\beta_{t^*}\gamma_{t^*}}\geq \frac{C\, |\Rbar_0(S_0)|}{\epsilon^2},
}$
with~${C = 8 / \log(1 + \sigma^{-2})}$, then there exists a $t\leq t^*$ such that, with probability at least~${1-\delta}$, 
%
$
\Rbar_\epsilon(S_0) \subseteq \Shato_t \subseteq \Shatp_t\subseteq \Rbar_0(S_0)$.
\label{thm:safety_and_complteness}
\end{restatable}
\cref{thm:safety_and_complteness} guarantees that \goose is safe with high probability. Moreover, for any priority function $h$ in \cref{alg:se}, it upper bounds the number of measurements that \cref{alg:SMDP_priority} requires to explore the largest safely reachable region $\Rbar_\epsilon(S_0)$. Note that \goose only achieves this upper bound if it is required by the \gls{IML} oracle. In particular, the following is a direct consequence of \cref{thm:safety_and_complteness}:
\begin{restatable}{corollary}{CorMainResult}
\label{cor:MainResult}
Under the assumptions of~\cref{thm:safety_and_complteness}, let the \gls{IML} oracle be deterministic given the observations. Then there exists a set $S$ with $\Rbar_\epsilon(S_0) \subseteq S \subseteq \Rbar_0(S_0)$ so that $\goal_i = \mathcal{O}(S)$ for all $k \geq 1$.
\end{restatable}
That is, the oracle decisions $\goal_i$ that we end up evaluating are the same as those by an oracle that was given the safe set $S$ in \cref{cor:MainResult} from the beginning. This is true since the set $\Shato_t$ converges to this set $S$. Since \cref{thm:safety_and_complteness} bounds the number of safety evaluations by $t^*$, \cref{cor:MainResult}  implies that, up to $t^*$ safety evaluations, \goose retains the properties (e.g., no-regret) of the \gls{IML} oracle $\mathcal{O}$ over $S$.




\paragraph{Choice of heuristic} 
While our worst-case guarantees hold for \textit{any} heuristic, the empirical performance of \goose depends on this choice. 
We propose to use the graph structure directly and additionally define a positive cost for each edge between two nodes. For a given edge cost, we define $c(\s,\goal_k, \Shato_t)$ as the cost of the minimum-cost path from $\s$ to $\goal_k$ within the optimistic safe set $\Shato_t$, which is equal to $\infty$ if a path does not exist,  and we consider the priority $h(\s)=-c(\s,\goal_k, \Shato_t)$. Thus, the node $\s$ with the lowest-cost path to $\goal_k$ has the highest priority. This reduces the design of a general heuristic to a more intuitive weight assignment problem, where the edge costs determine the planned path for learning about $\goal_k$ (dashed line in \cref{fig:set_illustration}).
One option for the edge cost is the inverse mutual information between $\s$ and the suggestion $\goal_k$, so that the resulting paths contain nodes that are informative about $\goal_k$. Alternatively, having successive nodes in the path close to each other under the metric $d(\cdot, \cdot)$, so that they can be easily added to the safe set and eventually lead us to $\goal_k$,  can be desirable. Thus, increasing monotone functions of the metric $d(\cdot, \cdot)$ can be effective edge costs.

\section{Applications and Experiments}\label{sec:Experiments}

In this section, we introduce two safety-critical \gls{IML} applications, discuss the consequences of \cref{thm:safety_and_complteness} for these problems, and empirically compare \goose to stae-of-the-art competing methods.
In our experiments, we set $\beta_t=3$ for all $t\geq 1$ as suggested by \citet{turchetta2016safe}. This choice of $\beta_t$ ensures safety in practice, but leads to more efficient exploration than the theoretical choice in \cref{thm:safety_and_complteness} \citep{turchetta2016safe,wachi2018safe}. Moreover, since in practice it is hard to estimate the Lipschitz constant of an unknown function, in our experiments we use the confidence intervals to define the safe set and the expanders as suggested by \citet{berkenkamp2016bayesian}.

\subsection{Safe Bayesian optimization}
In safe BO we want to optimize the unknown function $f $ subject to the unknown safety constraint $q$, see \cref{sec:ProblemStatement}. In this setting, we aim to find the best input over the largest set we can hope to explore safely, $\Rbar_\epsilon(S_0)$. The performance of an agent is measured in terms of the $\epsilon$-safe regret $\argmax_{\s\in \Rbar_\epsilon(S_0)}f(\s) - f(\s_t)$ of not having evaluated the function at the optimum in $\Rbar_\epsilon(S_0)$.

We combine \goose with the unsafe \textsc{GP-UCB} \citep{srinivas2009gaussian} algorithm as an oracle. For computational efficiency, we do not use a fully connected graph, but instead connect decisions only to their immediate neighbors as measured by the kernel and assign equal weight to each edge for the heuristic $h$.
We compare \goose to \textsc{SafeOPT} \citep{sui2015safe} and \textsc{StageOPT} \citep{sui2018stagewise} in terms of \textit{$\epsilon$-safe average regret}. Both algorithms use safe exploration as a proxy objective, see \cref{fig:IllustrativeExample}.
 

\begin{figure}
    \begin{floatrow}
        \TopFloatBoxes
        \ffigbox[\FBwidth]{
                \includegraphics[scale=1.]{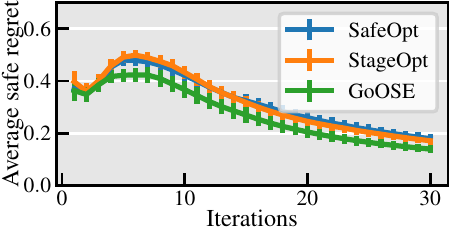}
                \includegraphics[scale=1.]{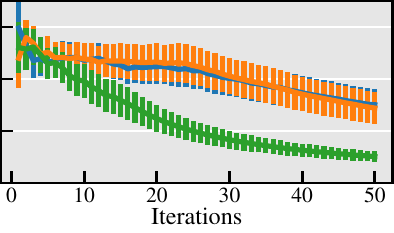}
            }
            {
            \caption{Average normalized $\epsilon$-safe regret for the safe optimization of GP samples over 40 (d=1, left) and 10 (d=2, right) samples. \goose only evaluates inputs that are relevant for the BO problem and, thereofore, it converges faster than its competitors.
            }
            \label{fig:SafeBO}
    }%
    \ttabbox[\Xhsize]{
        \begin{tabular}{l|c c}
             {} &\goose & \textsc{SEO} \\ \hline
             Sample & $\mathbf{30.0\,\%}$ & $38.4\,\%$ \\ \hline
             Cost & $12.7\,\%$ & $\mathbf{0.7}\,\%$ \\ \hline
             Time & $\mathbf{37.8}\,\%$ & $518\,\%$ 
        \end{tabular}
    }
    {
    \caption{Mars experiment performance normalized to \textsc{SMDP} in terms of 
    samples to find the first path, exploration cost and computation time per iteration.}
    \label{tab:baseline_comparison}
    }
    \end{floatrow}
\end{figure}

We optimize samples from a GP with zero mean and Radial Basis Function (RBF) kernel with variance $1.0$ and lengthscale $0.1$ and $0.4$ for a one-dimensional and two-dimensional, respectively. The observations are perturbed by i.i.d Gaussian noise with $\sigma=0.01$. For simplicity, we set the objective and the constraint to be the same, $f=q$. \cref{fig:SafeBO} (left) shows the average regret as a function of the number of evaluations $k+t$ averaged over 40 different samples from the GP described above over a one dimensional domain (200 points evenly distributed in  $[-1, 1]$). \cref{fig:SafeBO} (right) shows similar results averaged over 10 samples for a two dimensional domain ($25\times 25$ uniform grid in $[0,1]^2$).

These results confirm the intuition from \cref{fig:IllustrativeExample} that using safe exploration as a proxy objective reduces the empirical performance of safe BO algorithms. The impact is more evident in the two dimensional case where there are more points along the boundaries that are nor relevant to the optimization and that are evaluated for exploration purposes. 

\subsection{Safe shortest path in deterministic MDPs}

The graph that we introduced in \cref{sec:ProblemStatement} can model states (nodes) and state transitions (edges) in deterministic, discrete MDPs. Hence, \goose naturally extends to the goal-oriented safe exploration problem in these models.
We aim to find the minimum-cost safe path from a starting state $\start$ to a goal state $\goal$, without violating the unknown safety constraint, $q$. At best, we can hope to find the path within the largest safely learnable set $\Rbar_\epsilon(S_0)$ as in \cref{thm:safety_and_complteness} with cost $c(\start,\goal,\Rbar_\epsilon(S_0))$.

\paragraph{Algorithms} 
We compare \goose to \textsc{SEO} \citep{wachi2018safe} and \textsc{SMPD} \citep{turchetta2016safe}  in terms of samples required to discover the first path, total exploration cost and computation cost on synthetic and real world data.
The \textsc{SMDP} algorithm cannot take goals into account and serves as a benchmark for comparison. The \textsc{SEO} algorithm aims to safely learn a near-optimal policy for any given cost function and can be adapted to the safe shortest path problem by setting the cost to $c(\s)=-\|\s-\goal\|_1$. However, it cannot guarantee that a path to $\goal$ is found, if one exists. Since the goal $\goal$ is fixed, \goose does not need an oracle. 
For the heuristic we use and optimistic estimate of the cost of the safe shortest path from $\start$ to $\goal$ passing through $\s$; that is $ h_t(\s)=-\min_{\s'\in Pred(\s)}c(\start,\s,\Shatp_t)+\kappa c(\s,\goal,\Shato_t)$.
The first term is a conservative estimate of the safe optimal cost from $\start$ to the best predecessor of $\s$ in $\mathcal{G}$ and the second term is an optimistic estimate of the safe optimal cost from $\s$ to $\goal$ multiplied by $\kappa > 1$ to encourage early path discovery.
 Here, we use the predecessor node because $\phi_t(\s)=\infty$ for all $\s$ not in $\Shatp_t$. Notice that, if a safe path exists, \cref{thm:safety_and_complteness} guarantees that \goose finds the shortest one eventually.

\begin{figure*}
    \centering
    \begin{subfigure}[b]{0.32\textwidth}
        \includegraphics[scale=1]{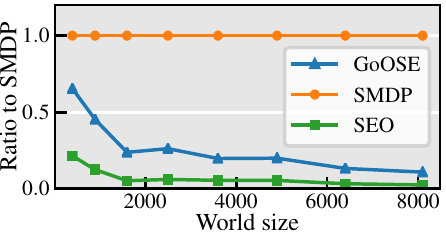}
        \caption{Cost of exploration.}
        \label{fig:traj_c}
    \end{subfigure}%
    \hfill%
    \begin{subfigure}[b]{0.32\textwidth}
        \includegraphics[scale=1]{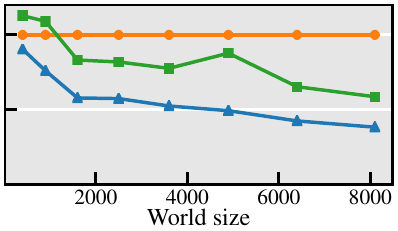}
        \caption{Samples to first path.}
        \label{fig:first_it}
    \end{subfigure}%
    \hfill%
    \begin{subfigure}[b]{0.32\textwidth}
        \includegraphics[scale=1]{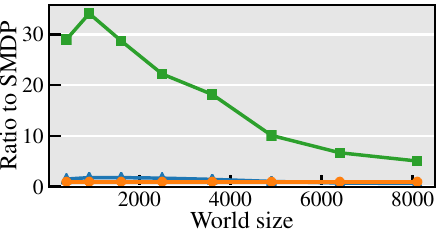}
        \caption{Computation per iteration.}
        \label{fig:time_per_it}
    \end{subfigure}%
    \caption{Performance of \goose  and \textsc{SEO} normalized to \textsc{SMDP} in terms of exploration cost, samples to find the first path and computation time per iteration as a function of the world size.
    }
    \label{fig:synthetic_experiments}
\end{figure*}

\paragraph{Synthetic data} 
Similarly to the setting in \cite{turchetta2016safe,wachi2018safe} we construct a two-dimensional grid world. At every location, the agent takes one of four actions: \textit{left}, \textit{right}, \textit{up} and \textit{down}. We use the state augmentation in \cite{turchetta2016safe} to define a constraint over state transitions. The constraint function is a sample from a GP with mean $\mu=0.6$ and RBF kernel with lengthscale $l=2$ and variance $\sigma^2=1$. If the agent takes an unsafe action, it ends up in a failure state, otherwise it moves to the desired adjacent state. We make the constraint independent of the direction of motion, i.e.,  $q(\s,\s')=q(\s',\s)$.
We generate 800 worlds by sampling 100 different constraints for square maps with sides of $20,30,40, \cdots, 90$ tiles and a source-target pair for each one. 

We show the geometric mean of the performance of \textsc{SEO} and \goose relative to \textsc{SMDP} as a function of the world size in \cref{fig:synthetic_experiments}. \cref{fig:first_it} shows that \goose needs a factor $2.5$ fewer samples than \textsc{SMDP}. 
\cref{fig:time_per_it} shows that the overhead to compute the heuristic of \goose is negligible, while the solution of the two MDPs \footnote{We use policy iteration. Policy evaluation is performed by solving a sparse linear system with SciPy~\citep{Scipy2019}. At iteration $t$, we initialize policy iteration with the optimal policy from  $t-1$.} required by \textsc{SEO} is computationally intense. 
\cref{fig:traj_c} shows that \textsc{SEO} outperforms \goose in terms of cost of the exploration trajectory. This is expected as \textsc{SEO} aims to minimize it, while \goose optimizes the sample-efficiency. However, it is easy to modify the heuristic of \goose to consider the exploration cost by, for example, reducing the priority of a state based on its distance from the current location of the agent.
In conclusion, \goose leads to a drastic improvement in performance with respect to the previously known safe exploration strategy with exploration guarantees, \textsc{SMDP}. Moreover, it achieves similar or better performance than \textsc{SEO} while providing exploration guarantees that \textsc{SEO} lacks.

\paragraph{Mars exploration} 
We simulate the exploration of Mars with a rover. In this context, communication delays between the rover and the operator on Earth make autonomous exploration extremely important, while the high degree of uncertainty about the environment requires the agent to consider safety constraints. In our experiment, we consider the \textit{Mars Science Laboratory} \citet[Sec. 2.1.3]{MSL2007MSL}, a rover deployed on Mars that can climb a maximum slope of $30^\circ$. We use Digital Terrain Models of Mars available from  
\citet{mcewen2007mars}.

We use a grid world similar to the one introduced above. The safety constraint is the absolute value of the steepness of the slope between two locations: given two  states $\s$ and $\s'$, the constraint over the state transition is defined as $q(\s, \s')=|H(\s)-H(\s')|/d(\s,\s')$, where $H(\s),H(\s')$ indicate the altitudes at $\s$ and $\s'$ respectively and $d(\s, \s')$ is the distance between them. We set conservatively the safety constraint to $q(\s, \s') \geq -\tan^{-1}(25^\circ)$. The step of the grid is  $10\mathrm{m}$.
We use square maps from 16 different locations on Mars with sides between 100 and 150 tiles. We generate 64 scenarios by sampling 4 source-target pairs for each map . We model the steepness with a GP with  Mat\'{e}rn kernel with $\nu=5/2$. 
We set the hyperprior on the lengthscale and on the standard deviation to be $Lognormal(30\mathrm{m}, 0.25\mathrm{m}^2)$ and $Lognormal(\tan(10^\circ), 0.04)$, respectively.
These encode our prior belief about the surface of Mars. Next, we take $1000$ noisy measurements at random locations from each map, which, in reality, could come from satellite images, to find a maximum a posteriori estimator of the hyperparameters to fine tune our prior to each  site. 

In Tab. 1, we show the geometric mean of the performance of \textsc{SEO} and \goose relative to \textsc{SMDP}. The results confirm those of the synthetic experiments but with larger changes in performance with respect to the benchmark due to the increased size of the world. 



\section{Conclusion} 
\label{sec:Conclusion}

We presented \goose, an add-on module that enables existing interactive machine learning algorithms to safely explore the decision space, without violating \textit{a priori} unknown safety constraints. Our method is provably safe and learns about the safety of decisions suggested by existing, unsafe algorithms. As a result, it is more data-efficient than previous safe exploration methods in practice. 

\paragraph{Aknowlegment.} This research was partially supported by the Max Planck ETH Center for Learning Systems and by the European Research Council (ERC) under the European Union’s Horizon 2020 research and innovation programme grantagreement No 815943.

\clearpage
\bibliographystyle{plainnat}
\bibliography{references.bib}

\clearpage
\appendix

In the following we present the proof of our result. 

\section{Definitions}\label{app:definitions}
For ease of consultation we repeat the relevant definitions here. We denote with $\mathcal{G}=(\dom,E)$ the directed graph describing the dependency among decisions introduced in \cref{sec:ProblemStatement}, where $\dom$ indicates the vertices of the graph, i.e., the decision space of the problem and $E\subseteq \dom\times \dom$ denotes the edges.
Baseline for safety:
\begin{equation} \label{app_def:Rsafe}
\Rsafe_\epsilon(S) = S \cup \{\s \in \dom \setminus S,|\, \exists \z \in S\colon q(\z)-\epsilon- L d(\s,\z) \geq 0\},
\end{equation}
The ergodicity operator is defined by intersecting the nodes that are reachable from a set $S$ and the nodes from which we can return to a set $\overline{S}$ through a path contained in another set $S$:
\begin{align}
    &\Rreach(S) = S \cup \{\s \in \dom \setminus S \,\vert\, \exists \z\in S: (\z,\s) \in E\}, \label{app_def:Rreach}\\
    &\Rreach_n(S) = \Rreach(\Rreach_{n-1}(S)) ~\textrm{with}~ \Rreach_1(S) = \Rreach(S), \\
    &\Rbreach(S) = \lim_{n \to \infty} \Rreach_n(S) \label{app_def:Rbreach}, \\
    &\Rret(S, \overline{S}) = \overline{S} \cup \{\s \in S \,\vert\, \exists \z \in \overline{S}: (\s,\z) \in E \}\\ \label{app_def:Rret}
    &\Rret_n(S, \overline{S}) = \Rret(S, \Rret_{n-1}(S, \overline{S})),
    ~\textrm{with}~ \Rret_1(S, \overline{S}) = \Rret(S, \overline{S}),\\
    &\Rbret(S, \overline{S}) = \lim_{n \to \infty} \Rret_n(S, \overline{S}) \label{app_def:Rbret},\\
    &\Rerg(S, \overline{S}) = \Rbreach(\overline{S}) \cap \Rbret(S, \overline{S}). \label{app_def:Rerg}
\end{align}
Here, we repeat the definition of the safe and ergodic  baseline introduced by \cite{turchetta2016safe}:
\begin{align}
    &R_{\epsilon}(S) = \Rsafe_\epsilon(S) \cap \Rerg(\Rsafe_\epsilon(S),\, S),\label{app_def:Repsilon}\\
    &R^n_{\epsilon}(S) = R_{\epsilon}(R^{n-1}_{\epsilon}(S)) ~\textrm{with} ~R^1_{\epsilon} = R_{\epsilon}(S),\label{app_def:RepsilonN}\\
    &\overline{R}_{\epsilon}(S)= \lim_{n\to\infty} R^n_{\epsilon}(S). \label{app_def:Rbar}
\end{align}
Optimistic and pessimistic constraint satisfaction operators:
\begin{align} 
    &o_t^\epsilon(S) = \{\s \in \dom, \,\vert\, \exists \z \in S:u_t(\z)-L d(\s,\z) - \epsilon\geq 0\} \label{app_def:g_operator},\\
    &p_t^\epsilon(S) = \{\s \in \dom, \,\vert\, \exists \z \in S:l_t(\z)-L d(\s,\z) - \epsilon\geq 0\}
\end{align}
Optimistic expansion operator:
\begin{align}
    &O_t^{\epsilon,1}(S) =o_t^\epsilon(S) \cap \Rerg(o_t^\epsilon(S), S),\\
    &O_t^{\epsilon, n}(S) = o_t^\epsilon(O_t^{\epsilon,n-1}(S)) \cap \Rerg(o_t^\epsilon(O_t^{\epsilon,n-1}(S)), S),\\
    &\tilde{O}_t^\epsilon(S) = \lim_{n \to \infty} O_t^{n,\epsilon}(S)
\end{align}
Pessimistic expansion operators:
\begin{align}
    &P_t^{\epsilon,1}(S) =p_t^\epsilon(S) \cap \Rerg(p_t^\epsilon(S), S),\\
    &P_t^{\epsilon, n}(S) = p_t^\epsilon(P_t^{\epsilon,n-1}(S)) \cap \Rerg(p_t^\epsilon(P_t^{\epsilon,n-1}(S)), S),\\
    &\tilde{P}_t^\epsilon(S) = \lim_{n \to \infty} P_t^{n,\epsilon}(S)
\end{align}
Pessimistic and optimistic safe and ergodic sets:
\begin{align}
    &\Shato_t = \tilde{O}_t^\epsilon(\Shatp_{t-1}) \label{app_def:Shato}\\
    &\Shatp_t = \tilde{P}_t^0(\Shatp_{t-1}),\label{app_def:Shatp}
\end{align}
Points with uncertainty above threshold:
\begin{equation} \label{app_def:W_set}
    W_t^\epsilon = \{\s \in \Shatp_t: ~w_t(\s) > \epsilon\}.
\end{equation}
Set of decisions with equal priority:
\begin{equation}
    \Sadj_t(\priority) = \{ \s\in \Shato_t\setminus p_t^0(, \Shatp_{t}):~ h_t(\s)=\priority \}
\end{equation}
Immediate expanders for nodes with priority $\priority$:
\begin{equation}
    G_t^\epsilon(\priority)=\{\s \in W_t^\epsilon, \mid \exists\, \z \in \Sadj_t(\priority) \colon u_t(\s) - L d(\s, \z) \geq 0 \}.,   
\end{equation}
Relevant priority:
\begin{equation}\label{app_def:p}
\priority^* = \max ~\priority, ~\textrm{s.t.}~~ |G_t^\epsilon(\priority)| >0.
\end{equation}
%

Notice that our definition of $\Shatp_t$ differs slightly from the one in \cite{turchetta2016safe} in that we alternate the pessimistic expansion step and the restriction to ergodic nodes until convergence, whereas \cite{turchetta2016safe} do it only once at each time step. In particular, we have $\Shatp_t=\lim_{n\to\infty} P^{0,n}_t(\Shatp_{t-1})$, while they use $\Shatp_t= P^{0,1}_t(\Shatp_{t-1})$. In practice, this does not make any difference since it is easy to verify that, by $t^*$, i.e., the time by when both \goose and the approach in \cite{turchetta2016safe} are guaranteed to converge in the worst case, the pessimistic ergodic safe sets should be the same for both methods. However, our new definition allows for a more efficient exploration. These new definitions would require us to show again some of the properties that were shown by \cite{turchetta2016safe} for $\Shatp$. However, due to our recursive definition of $\Shatp$, it is easy to see that it is possible to show these properties by induction over the index $n$. In this case, the lemmas introduced by \cite{turchetta2016safe} constitute the base case. At this point, it is sufficient to use the induction hypothesis and the monotonicity of the confidence interval shown in \cref{lem:monotonicity} together with basic properties of the $\Rerg$ operator discussed in \cref{lem:RergSubset} to prove the induction step. We  show how to do this explicitly in \cref{lem:NonDecresingS}. However, we do not explicitly repeat this reasoning for every lemma involving $\Shatp$ and we refer to \cite{turchetta2016safe} instead.

\section{Preliminary lemmas}
This section contains some basic lemmas about the sets defined above that will be used in subsequent sections to prove our main results.
\begin{restatable}{lemma}{monotonicity}
\label{lem:monotonicity}
\,$\forall \s \in \dom,\,u_{t+1}(\s)\leq u_t(\s),\,l_{t+1}(\s)\geq l_t(\s),\, w_{t+1}(\s)\leq w_t(\s)$.
\end{restatable}
\begin{proof}
See Lemma 1 in \cite{turchetta2016safe}.
\end{proof}
%
%
%
%
%
%
\begin{lemma} \label{lem:Rbar_ret_subset}
Given $S \subseteq R \subseteq \dom$ and $\overline{S} \subseteq \overline{R} \subseteq \dom$, it holds that $\Rbret(\overline{S},S)\subseteq \Rbret(\overline{R},R)$.
\end{lemma}
\begin{proof}
See Lemma 7 in \cite{turchetta2016safe}
\end{proof}
\begin{restatable}{lemma}{Rreach_nSubset}
\label{lem:Rreach_nSubset}
    For any $S, R \subseteq \dom$, for any $n \geq 1$ we have $\Rreach_n(S) \subseteq \Rreach_n(R)$.
\end{restatable}
We proceed by induction. For $n=1$, we have $\Rreach(S) \subseteq \Rreach(R)$ by Lemma 8 by \cite{turchetta2016safe}. For the inductive step, assume $\Rreach_{n-1}(S) \subseteq \Rreach_{n-1}(R)$. Consider $\s \in \Rreach_n (S)$. We know $\exists \s' \in \Rreach_{n-1}(S) \subseteq \Rreach_{n-1}(R), \, a\in \mathcal{A}(\s')$ such that $\s = f(\s', a)$, which implies $\s \in \Rreach_n (R)$.
\begin{restatable}{corollary}{RbreachSubset}
\label{cor:RbreachSubset}
    For any $S, R \subseteq \dom$, we have $\Rbreach(S) \subseteq \Rbreach(R)$.
\end{restatable}
\begin{restatable}{lemma}{RergSubset}
\label{lem:RergSubset}
    For any $S, R \subseteq \dom$ and $\overline{S}, \overline{R} \subseteq \dom$,  we have $\Rerg(S, \overline{S}) \subseteq \Rerg (R, \overline{R})$.
\end{restatable}
\begin{proof}
    This follows from \cref{lem:Rbar_ret_subset} and \cref{cor:RbreachSubset}.
\end{proof}
\begin{restatable}{lemma}{NonDecreasingS}
\label{lem:NonDecresingS}
For any $t\geq 1$, $\Shatp_0\subseteq \Shatp_t \subseteq \Shatp_{t+1}$.
\end{restatable}
\begin{proof}
Lemma 9 in \cite{turchetta2016safe} allows us to say $P_t^{0,1}(\Shatp_{t-1})\subseteq P_{t+1}^{0,1}(\Shatp_{t})$. Thus, we can assume $P_t^{0,n-1}(\Shatp_{t-1})\subseteq P_{t+1}^{0,n-1}(\Shatp_{t})$ as induction hypothesis. Let us consider $\s \in P_t^{0,n}(\Shatp_{t-1})$. We know there exists $\z\in P_t^{0,n-1}(\Shatp_{t-1}) \subseteq P_{t+1}^{0,n-1}(\Shatp_{t})$ such that $l_t(\z) - Ld(\s,\z) \geq 0$, which ,by \cref{lem:monotonicity}, implies $l_{t+1}(\z) - Ld(\s,\z) \geq 0$. This means that $p_t^0(P^{0,n-1}_t(\Shatp_{t-1})) \subseteq p_{t+1}^0(P^{0,n-1}_{t+1}(\Shatp_{t}))$. Applying \cref{lem:RergSubset}, we complete the induction step and show $P_t^{0,n}(\Shatp_{t-1}) \subseteq P_{t+1}^{0,n}(\Shatp_{t})$.
\end{proof}
%
%
%
%
%
%
%
%
%
%
%
%
\begin{restatable}{lemma}{ConfidenceInterval}
\citep[Thm. 2]{chowdhury2017kernelized}
\label{ConfidenceInterval}
Assume ~${\| q \|_k^2 \leq B_q}$, and $\sigma$-sub-Gaussian noise. If~$\beta_t^{1/2}=B_q+4\sigma \sqrt{\gamma_t + 1 + \mathrm{ln}(1/\delta)}$, then, for all~${t > 0}$ and all~$\s \in \dom$,~$|q(\s)-\mu_{t-1}(\s)|\leq \beta^{1/2}_t\sigma_{t-1}(\s)$ holds with probability at least~${1 - \delta}$.
\end{restatable} 
\begin{proof}
See Theorem 2 in \cite{chowdhury2017kernelized}.
\end{proof}
%

\begin{restatable}{lemma}{ChoiceOfBeta}
Let~$\beta_t^{1/2}=B_q+4\sigma \sqrt{\gamma_t + 1 + \mathrm{ln}(1/\delta)}$ and assume ~${\| q \|_k^2 \leq B_q}$, and $\sigma$-sub-Gaussian noise. Then, for all~${t > 0}$ and all~$\s \in \dom$, it holds with probability at least~${1 - \delta}$ that~$q(\s) \in C_t(\s)$.
\label{thm:beta}
\end{restatable}

\begin{proof}
See Corollary 1 in \cite{sui2015safe}.
\end{proof}
%
%
%


\section{Safety}
The safety of our algorithm depends on the confidence intervals and on the safe and ergodic set $\Shatp_t$. Since these are defined as in \cite{turchetta2016safe}, their safety guarantees carry over to our case.
\begin{restatable}{theorem}{Safety}
\label{thm:Safety_of_trajectory}
For any node $\s$ along any trajectory induced by \cref{alg:SMDP_priority} on the graph $\mathcal{G}$ we have, with probability at least $1-\delta$, that $q(\s)\geq 0$.
\end{restatable}
\begin{proof}
See Theorem 2 in \cite{turchetta2016safe}.
\end{proof}


\section{Completeness}\label{app:completeness}
In this section, we develop the core of our theoretical contribution. The analysis in \cite{turchetta2016safe} bounds the uncertainty of the expanders when the safe set does not change in an interval $[t_0, t_1]$ without considering the measurements collected prior to $t_0$. By considering this information, we extend their worst case sample complexity bound to our more general formulation of the safe exploration problem.

The following lemmas refer to the exploration steps, i.e., when the goal suggested by the oracle $\mathcal{O}$ lies outside of the pessimistic safe and ergodic set (\cref{alg:goose:line:while_unsafe}, \cref{alg:goose}). Notice that $t$ denotes the number of constraint evaluations and it differs from the iteration index of the algorithm $i$.

The core idea is the following: We bound the number constraint evaluations required at point in the domain to guarantee that its uncertainty is below $\epsilon$. We show that, as a consequence, if the safe and ergodic set does not change for long enough all the expanders have uncertainty below $\epsilon$. At this point we can either guarantee that the safe set expands or that the whole $\Rbar_\epsilon(S_0)$ has been explored. Since the analysis relies on the number of constraint evaluations at each point in the domain, we can evaluate them in any order as long as we exclude those that have an uncertainty below $\epsilon$. Therefore, our exploration guarantees hold for any priority function.

In the following, let us denote with $\mathcal{T}_t^\s=\{\tau_1, \cdots , \tau_j\}$ the set of steps where the constraint $q$ is evaluated at $\s$ by step $t$. Moreover, we assume, without loss of generality, $k(\s,\s)\leq 1$, i.e., we assume bounded variance.
\begin{restatable}{lemma}{BoundUncertaintyWithNumberOfSamples}
\label{lem:BoundUncertaintyWithNumberOfSamples}
For any $t\geq 1$ and for any $\s \in \dom$, it holds that $w_t(\s) \leq \sqrt{\frac{C_1 \gamma_t \beta_t}{|\mathcal{T}_t^\s|}}$, with $C_1 = 8/\log(1-\sigma^{-2})$.
\end{restatable}
\begin{proof}
    \begin{align}
        |\mathcal{T}_t^\s| w^2_t(\s) & \leq \sum_{\tau \in \mathcal{T}_t^\s} w^2_\tau(\s) \label{eq1_lem_BoundUncertaintyWithNumberOfSamples}\\
        & \leq \sum_{\tau \in \mathcal{T}_t^\s} 4 \beta_\tau \sigma^2_{\tau-1}(\s), \label{eq2_lem_BoundUncertaintyWithNumberOfSamples}\\
        & \leq \sum_{\tau \leq t} 4 \beta_\tau \sigma^2_{\tau-1}(\s), \label{eq3_lem_BoundUncertaintyWithNumberOfSamples} \\ 
        & \leq  C_1 \gamma_t \beta_t, \label{eq4_lem_BoundUncertaintyWithNumberOfSamples}
    \end{align}
    with $C_1 = 8/\log(1-\sigma^{-2})$. Here, \cref{eq1_lem_BoundUncertaintyWithNumberOfSamples} holds because of \cref{lem:monotonicity}, and \cref{eq4_lem_BoundUncertaintyWithNumberOfSamples} holds because of Lemma 5.4 by \citet{srinivas2009gaussian}.
\end{proof}
\todo{with this $\beta$ things are not correct at the moment}
For the remainder of the paper, let us denote with $T_t$ the smallest positive integer such that $\frac{T_t}{\beta_{t+T_t}\gamma_{t+T_t}}\geq \frac{C_1}{\epsilon^2}$, with $C_1 = 8/\log(1-\sigma^{-2})$ and with $t^*$ the smallest positive integer such that $t^* \geq |\Rbar_0(S_0)|T_{t^*}$.
\begin{restatable}{lemma}{NumberOfSamples}
\label{lem:NumberOfSamples}
    For any $t \leq t^*$, for any $\s \in \dom$ such that $|\mathcal{T}^\s_t| \geq T_{t^*}$, it holds that $w_t(\s) \leq \epsilon$.
\end{restatable}
\begin{proof}
    Since $T_t$ is an increasing function of $t$ \citep{sui2015safe}, we have $|\mathcal{T}^\s_t| \geq T_{t^*} \geq T_t$. Therefore, using \cref{lem:BoundUncertaintyWithNumberOfSamples} and the definition of $T_t$, we have 
    \begin{equation}
        w_t(\s) \leq \sqrt{\frac{C_1 \gamma_t \beta_t}{T_t}} \leq \sqrt{\frac{C_1 \gamma_t \beta_t \epsilon^2}{C_1\gamma_{t+T_t} \gamma_{t+T_t}}} \leq \sqrt{\frac{\gamma_t \beta_t}{\gamma_{t+T_t} \gamma_{t+T_t}}}~ \epsilon \leq \epsilon,
    \end{equation}
    where the last inequality comes from the fact that both $\beta_t$ and $\gamma_t$ are increasing functions of $t$.
\end{proof}
\begin{restatable}{lemma}{UpperBoundNumberOfSamples}
\label{lem:UpperBoundNumberOfSamples}
    For any $t \leq t^*$, $|\mathcal{T}_t^\s| \leq T_{t^*}$, for any $\s \in \Shatp_t$.
\end{restatable}
\begin{proof}
    According to Line \ref{alg:se:line:uncertainty_sampling} of \cref{alg:SE}, we only evaluate the constraint at points $\s\in \dom$ if $w_t(\s) > \epsilon$. From \cref{lem:NumberOfSamples} we know that $|\mathcal{T}_t^\s| = T_{t^*} \implies w_t(\s) \leq \epsilon$. Thus, if $|\mathcal{T}_t^\s|=T_{t^*}$, $\s$ is not evaluated anymore, which means that $|\mathcal{T}_t^\s|$ cannot grow any further.
\end{proof}
\begin{lemma} \label{lem:Superset}
$\forall t\geq 0,\,\Shatp_{t}\subseteq \overline{R}_{0}(S_0)$ with probability at least $1-\delta$.
\end{lemma}
\begin{proof}
See Lemma 22 in \cite{turchetta2016safe}.
\end{proof}

The following lemma bounds the uncertainty of the points sampled by \textsc{GoOSE} when the set of safe and ergodic points does not increase.
\begin{restatable}{lemma}{UncertaintyInExpanders}
\label{lem:UncertaintyInExpanders}
    For any $t \leq t^*$, let $\tau_t= |\Shatp_t|T_{t^*}$, if $\Shatp_t=\Shatp_{\tau_t}$, then $w_{\tau_t}(\s) \leq \epsilon$ for all $\s \in \cup_\priority G_{\tau_t}^\epsilon(\priority)$. 
\end{restatable}
\begin{proof}
    First, we notice that
         \begin{equation}
             \sum_{\s \in \Shatp_{\tau_t}}|\mathcal{T}_{\tau_t}^\s| = \tau_t = |\Shatp_t|T_{t^*} = |\Shatp_{\tau_t}|T_{t^*},
         \end{equation} 
    where the first equality comes from the fact that the sum of the number of constraint observations in $\tau_t$ time steps is equal to $\tau_t$, the second comes from the definition of $\tau_t$ and the third comes from the assumption that  $\Shatp_t=\Shatp_{\tau_t}$.
    This allows us to say that, for all $\s \in \Shatp_{\tau_t}$ 
    \begin{equation}
        \sum_{\z \in \Shatp_{\tau_t}\setminus \{\s\}} |\mathcal{T}_{\tau_t}^\z|=|\Shatp_{\tau_t}|T_{t^*} - |\mathcal{T}_{\tau_t}^\s|. \label{lem:UncertaintyInExpanders:eq1}
    \end{equation}
    Moreover, we have $\tau_t=|\Shatp_t|T_{t^*}\leq |\Rbar_0(S_0)|T_{t^*}\leq t^*$ by definition of $t^*$ and $\tau_t$ and \cref{lem:Superset}.
    Therefore, by \cref{lem:UpperBoundNumberOfSamples} we know that $T_{t^*} \geq |\mathcal{T}_{\tau_t}^\s|$ for all $\s \in \Shatp_{\tau_t}$. Now we show by contradiction that $T_{t^*} = |\mathcal{T}_{\tau_t}^\s|$ for all $\s \in \Shatp_{\tau_t}$. Assume this is not the case, i.e., there is $\s \in \Shatp_{\tau_t}$ such that $T_{t^*} > |\mathcal{T}_{\tau_t}^\s|$. We have
    \begin{equation}
        (|\Shatp_{\tau_t}| - 1) T_{t^*} \geq \sum_{\z \in \Shatp_{\tau_t}\setminus \{\s\}} |\mathcal{T}_{\tau_t}^\z| = |\Shatp_{\tau_t}|T_{t^*} - |\mathcal{T}_{\tau_t}^\s| > |\Shatp_{\tau_t}|T_{t^*} - T_{t^*}= (|\Shatp_{\tau_t}| - 1) T_{t^*},
    \end{equation}
    which is a contradiction and proves our claim that $T_{t^*} = |\mathcal{T}_{\tau_t}^\s|$ for all $\s \in \Shatp_{\tau_t}$. Therefore, by \cref{lem:NumberOfSamples}, $w_{\tau_t}(\s) \leq \epsilon$ for all $\s \in \Shatp_{\tau_t}$. This proves our claim since $\cup_\priority G_{\tau_t}^\epsilon(\priority) \subseteq \Shatp_{\tau_t}$.
\end{proof}

\begin{restatable}{lemma}{ShortTermExploration}
\label{lem:ShortTermExploration}
For any $t\geq 1$, $\overline{R}_{\epsilon}(S_0)\setminus \Shatp_t \neq \emptyset$, then, $R_{\epsilon}(\Shatp_t)\setminus \Shatp_t \neq \emptyset$.
\end{restatable}
\begin{proof}
See Lemma 20 in \cite{turchetta2016safe}.
\end{proof}
\begin{restatable}{lemma}{Expansion}
\label{lem:ExpansionWithExpanders}
For any $t \leq t^*$, if $\Rbar_{\epsilon}(S_0) \setminus \Shatp_t \neq \emptyset$, then $\Shatp_t \subset \Shatp_{|\hat{S}_t|T_{t^*}}$ with probability at least $1-\delta$.
\end{restatable}
\begin{proof}
This proof is analogous to Lemma 21 in \cite{turchetta2016safe} where we use our  \cref{lem:UncertaintyInExpanders} rather than their Lemma 19 to bound the uncertainty of the expanders.
\end{proof}
\begin{restatable}{lemma}{ContainREpsilon}
\label{lem:ContainREpsilon}
There exists $t \leq t^*$ such that $\Rbar_{\epsilon}(S_0) \subseteq \Shatp_{t}$ with probability at least $1 - \delta$.
\end{restatable}
\begin{proof}
For the sake of contradiction, assume this is not the case and that $\forall t \leq t^*$ holds that $\Rbar_{\epsilon}(S_0) \setminus \Shatp_{t} \neq \emptyset$. For all $i \geq 1$ define $\tau_i = |\Shatp_{\tau_{i-1}}|T_{t^*}$ with $\tau_0 = 0$. We know that $\tau_i \leq t^*$ for all $i$ because of \cref{lem:Superset} and that $\tau_0\leq \tau_1 \leq \cdots$ because of \cref{lem:NonDecresingS}. Therefore, \cref{lem:ExpansionWithExpanders} implies that $\Shatp_0 \subset \Shatp_{\tau_1} \subset \Shatp_{\tau_2}\subset \cdots$.
In general, this means that $|\Shatp_{\tau_i}| \geq |\Shatp_0| + i$ for all $i \geq 1$. In particular if we set $i=|\Rbar_0(S_0) \setminus S_0| + 1$ we get that $|\Shatp_{\tau_i}| \geq |\Shatp_0| + |\Rbar_0(S_0) \setminus S_0| + 1= |\Rbar_0(S_0)| + 1 > |\Rbar_0(S_0)|$. This is a contradiction because of \cref{lem:Superset}.
\end{proof}
\begin{restatable}{lemma}{Completeness}
\label{thm:Completeness}
   There is $t\leq t^*$ such that $\Rbar_\epsilon(S_0) \subseteq \Shatp_{t} \subseteq \overline{R}_0(S_0)$ with probability at least \,$1-\delta$.
\end{restatable}
\begin{proof}
    The lemma follows directly from \cref{lem:Superset,lem:ContainREpsilon}.
\end{proof}
\begin{restatable}{lemma}{SpessSoptConvergence}
\label{lem:SoptConvergence}
    $\Shato_{t^*}\subseteq\Shatp_{t^*}$.
\end{restatable}
\begin{proof}
    Since the optimistic and ergodic and safe set is defined recursively, we will prove this claim by induction.
    Similarly to \cref{lem:UncertaintyInExpanders}, we start by noticing that, for every $\s \in \Shatp_{t^*}$, we have:
    \begin{align}
        &\sum_{\s \in \Shatp_{t^*}}|\mathcal{T}_{t^*}^\s| = t^* \geq |\Rbar_0(S_0)|T_{t^*},\\
        \implies & \sum_{\z \in \Shatp_{t^*} \setminus \s}|\mathcal{T}_{t^*}^\z| \geq |\Rbar_0(S_0)|T_{t^*} - |\mathcal{T}_{t^*}^\s| \geq |\Shatp_{t^*}|T_{t^*} - |\mathcal{T}_{t^*}^\s|.
    \end{align}
    \cref{lem:UpperBoundNumberOfSamples} allows us to say $|\mathcal{T}_{t^*}^\s|\leq T_{t^*}$ for all $\s \in \Shatp_{t^*}$. We show by contradiction that $|\mathcal{T}_{t^*}^\s|= T_{t^*}$ for all $\s \in \Shatp_{t^*}$. Assume this is not the case and that we have $\s\in \Shatp_{t^*}$ such that $|\mathcal{T}_{t^*}^\s| < T_{t^*}$:
    \begin{equation}
        (|\Shatp_{t^*}| - 1) T_{t^*} \geq \sum_{\z \in \Shatp_{t^*}\setminus \{\s\}} |\mathcal{T}_{t^*}^\z| \geq |\Shatp_{t^*}|T_{t^*} - |\mathcal{T}_{t^*}^\s| > |\Shatp_{t^*}|T_{t^*} - T_{t^*}= (|\Shatp_{t^*}| - 1) T_{t^*},
    \end{equation}
    which is a contradiction and proves that $|\mathcal{T}_{t^*}^\s|= T_{t^*}$ for all $\s \in \Shatp_{t^*}$. Therefore, by \cref{lem:NumberOfSamples}, we know that $w_{t^*}(\s) \leq \epsilon$ for any $\s \in \Shatp_{t^*}$.
    Now consider $\s \in o_{t^*}^\epsilon(\Shatp_{t^*-1})$. We know that there is $\z\in \Shatp_{t^*-1}\subseteq \Shatp_{t^*}$ such that $u_{t^*}(\z)-Ld(\s,\z) - \epsilon \geq 0$. Since $w_{t^*}(\z) \leq \epsilon$, we know that $l_{t^*}(\z)-Ld(\s,\z) \geq 0$, i.e., $\s \in p_{t^*}^0(\Shatp_{t^*-1})$. Using \cref{lem:RergSubset}, we can say $O^{\epsilon,1}_{t^*}(\Shatp_{t^*-1}) \subseteq P^{0,1}_{t^*}(\Shatp_{t^*-1})$. 
    Now, we can make the following induction hypothesis: $O^{\epsilon,n-1}_{t^*}(\Shatp_{t^*-1}) \subseteq P^{0,n-1}_{t^*}(\Shatp_{t^*-1})$.
    Consider $\s \in O^{\epsilon}_{t^*}(\Shatp_{t^*-1})$, we know there is $\z\in O^{\epsilon,n-1}_{t^*}(\Shatp_{t^*-1})\subseteq P^{0,n-1}_{t^*}(\Shatp_{t^*-1}) \subseteq \Shatp_{t^*}$ (where the first inclusion comes from the induction hypothesis and the second by definition of the safe set), such that $u_{t^*}(\z)-Ld(\s,\z) - \epsilon \geq 0$. Since $w_{t^*}(\z) \leq \epsilon$, we know that $l_{t^*}(\z)-Ld(\s,\z) \geq 0$, i.e., $\s \in p_{t^*}^0(P^{0,n-1}_{t^*}(\Shatp_{t^*-1}))$. We can apply \cref{lem:RergSubset} again to complete the induction step and show $O^{\epsilon,n}_{t^*}(\Shatp_{t^*-1}) \subseteq P^{0,n}_{t^*}(\Shatp_{t^*-1})$ and, therefore, $\Shato_{t^*}\subseteq \Shatp_{t^*}$.
\end{proof}
\begin{restatable}{lemma}{ShatoContainsRbarEpsilon}
\label{lem:ShatoContainsRbarEpsilon}
    For every $t\geq 0$, we have $\Rbar_\epsilon(S_0) \subseteq \Shato_t$.
\end{restatable}
\begin{proof}
    We will show this claim with a proof by induction. Let us consider $\s \in R_\epsilon(S_0)$. We know there is a $\z \in S_0$ such that $q(\z) - Ld(\s,\z) - \epsilon \geq 0$. By \cref{thm:beta}, we know this means $u_t(\z) - Ld(\s,\z) - \epsilon \geq 0$ for all $t\geq 0$. Therefore, $\Rsafe_\epsilon(S_0) \subseteq o_t^\epsilon(S_0) \subseteq o_t^\epsilon(\Shatp_{t-1})$ since $S_0\subseteq \Shatp_{t-1}$ for all $t\geq 1$ by \cref{lem:NonDecresingS}. \cref{lem:RergSubset} allows us to say that $R_\epsilon(S_0) \subseteq O^\epsilon_t(\Shatp_{t-1})$. As induction hypothesis, we can assume $R^{n-1}_\epsilon(S_0) \subseteq O^{\epsilon, n-1}_t(\Shatp_{t-1})$. Consider $\s \in R^{n}_\epsilon(S_0)$. We know there is $\z \in R^{n-1}_\epsilon(S_0) \subseteq O^{\epsilon, n-1}_t(\Shatp_{t-1})$ such that $q(\z) - Ld(\s,\z) - \epsilon \geq 0$ which, by \cref{thm:beta}, means that $u_t(\z) - Ld(\s,\z) - \epsilon \geq 0$ for all $t\geq 0$. Therefore, $\Rsafe_\epsilon(R^{n-1}_\epsilon(S_0)) \subseteq o^\epsilon_t(O^{\epsilon, n-1}_t(\Shatp_{t-1}))$. Using \cref{lem:RergSubset}, we can say $R^n_\epsilon(S_0) \subseteq O^{\epsilon, n}_t(\Shatp_{t-1})$, which completes the induction step and concludes the proof.
\end{proof}
%
%
%


\section{Main result}\label{app:main_result}
\mainresult*
\begin{proof}
    The safety is a direct consequence of \cref{thm:Safety_of_trajectory}.
    The convergence of the pessimistic and optimistic approximation of the safe sets is a direct consequence of \cref{thm:Completeness,lem:SoptConvergence,lem:ShatoContainsRbarEpsilon}.
\end{proof}

The following corollary gives a simpler interpretation of our main result: in presence of an unknown constraint, the IML oracle augmented with \goose behaves as the IML oracle would behave if it had knowledge of a better-than-$\epsilon$-accurate approximation of the safe reachable set from the beginning (except for a finite number of constraint evaluations).
\CorMainResult*
\begin{proof}
    This is a direct consequence of \cref{thm:safety_and_complteness}, since it guarantees that, if necessary, we can expand the set of points where we can evaluate the objective (i.e. the pessimistic safe set) and we can contract the decision space of the IML oracle (i.e. the optimistic safe set) to a point where the first contains the second, in a finite number of constraint evaluations. 
\end{proof}



\end{document}